\documentclass{article}

\usepackage{microtype}
\usepackage{graphicx}
\usepackage{subcaption}
\usepackage{booktabs} % for professional tables
\usepackage{multirow} % for multirow tables

\usepackage{hyperref}

\usepackage[preprint]{icml2026}

\usepackage{amsmath}
\usepackage{amssymb}
\usepackage{mathtools}
\usepackage{amsthm}

\usepackage[capitalize,noabbrev]{cleveref}

\theoremstyle{plain}
\newtheorem{theorem}{Theorem}[section]

\newtheorem{lemma}[theorem]{Lemma}
\newtheorem{corollary}[theorem]{Corollary}
\theoremstyle{definition}

\theoremstyle{remark}

\usepackage[textsize=tiny]{todonotes}

\icmltitlerunning{When Good Enough Is Optimal: Multiplication-Only Matrix Inversion Approximation for Quantized Gated DeltaNet}

\begin{document}

\twocolumn[
  \icmltitle{When Good Enough Is Optimal: Multiplication-Only Matrix Inversion Approximation for Quantized Gated DeltaNet}



  \icmlsetsymbol{equal}{*}

  \begin{icmlauthorlist}
    \icmlauthor{Luoming Zhang}{equal,yyy}
    \icmlauthor{Yuwei Ren}{equal,yyy}
    \icmlauthor{Kui Zhang}{yyy}
    \icmlauthor{Tian Liu}{yyy}
    \icmlauthor{Lingjuan Ge}{yyy}
    \icmlauthor{Denghao Li}{yyy}
    \icmlauthor{Matthew Harper Langston}{yyy}
    \icmlauthor{Yin Huang}{yyy}
    \icmlauthor{Weiliang Will Zeng}{yyy}
    \icmlauthor{Liang Zhang}{yyy}
  \end{icmlauthorlist}

  % \icmlaffiliation{equal}{Equal contribution}
  \icmlaffiliation{yyy}{Qualcomm AI Research, an initiative of Qualcomm Technologies, Inc}

  \icmlcorrespondingauthor{Luoming Zhang}{luomzhan@qti.qualcomm.com}
  \icmlcorrespondingauthor{Yuwei Ren}{ren@qti.qualcomm.com}

  \icmlkeywords{Machine Learning, ICML}

  \vskip 0.3in
]



\printAffiliationsAndNotice{\icmlEqualContribution}

\begin{abstract}

Matrix inversion in chunk-wise parallel linear attention is a major bottleneck for long-context modeling, particularly on NPUs, where forward-substitution-based methods exhibit limited parallelism and poor hardware utilization.
We propose a fast, Matrix Multiplication (MatMul)-based algorithm tailored for strictly lower-triangular matrices arising in chunk-wise linear attention.
Motivated by the rapid growth of Neumann-series terms and the diagonal concentration of the inverse matrix, we employ a truncated Neumann expansion with structural masking and parallel residual correction to eliminate sequential dependencies.
We further extend our method to low-bits INT by mitigating the dynamic range expansion arising from repeated matrix power operations, and adapt the approximation order and residual step to the chunk size to minimize computational cost while preserving the model's accuracy. 
Experiments on Qwen3.5-family models demonstrate up to 5$\times$ kernel-level speedup and a 20\% reduction in decode-layer overhead, while preserving accuracy under both floating-point and low-precision inference. Our method offers an efficient and hardware-friendly solution for scalable linear attention.

\end{abstract}

\section{Introduction}
% As the input length of large language models (LLMs) increases~\cite{DBLP:journals/corr/abs-2510-26692,qwen3.5,gemma4}, standard attention architectures~\cite{DBLP:conf/nips/VaswaniSPUJGKP17} increasingly face challenges in memory consumption and runtime efficiency. To address these challenges, recent work~\cite{DBLP:conf/icml/KatharopoulosV020,DBLP:conf/icml/HuaDLL22,DBLP:journals/corr/abs-2307-08621,DBLP:journals/corr/abs-2312-00752} has shifted toward optimizing linear attention architectures, which preserve a fixed-size state representation of past information and eliminate the quadratic $\mathcal{O}(T^2)$ memory and computational costs associated with conventional attention mechanisms. GatedDeltaNet~\cite{DBLP:conf/iclr/YangKH25} has emerged as a widely used linear attention architecture, with models such as QWen3.5~\cite{qwen3.5} and KiMi~\cite{DBLP:journals/corr/abs-2510-26692} demonstrating its strong capability when scaled to large model sizes. In GatedDeltaNet, historical information is compressed into a single state and updated token by token, resulting in limited parallelism in computation.

As LLM context lengths grow~\cite{DBLP:journals/corr/abs-2510-26692,qwen3.5,gemma4}, standard attention~\cite{DBLP:conf/nips/VaswaniSPUJGKP17} suffers from quadratic memory and runtime costs.  Linear attention methods~\cite{DBLP:conf/icml/KatharopoulosV020,DBLP:conf/icml/HuaDLL22,DBLP:journals/corr/abs-2307-08621,DBLP:journals/corr/abs-2312-00752} address this issue by maintaining a fixed-size recurrent state, avoiding the quadratic $\mathcal{O}(T^2)$ cost. GatedDeltaNet~\cite{DBLP:conf/iclr/YangKH25}, adopted by recent large-scale models such as QWen3.5~\cite{qwen3.5} and KiMi~\cite{DBLP:journals/corr/abs-2510-26692}, achieves strong long-context capability, but its sequential state update remains a key bottleneck for parallel hardware execution.

\begin{figure}[t]
  \begin{center}
    \centerline{\includegraphics[width=\columnwidth]{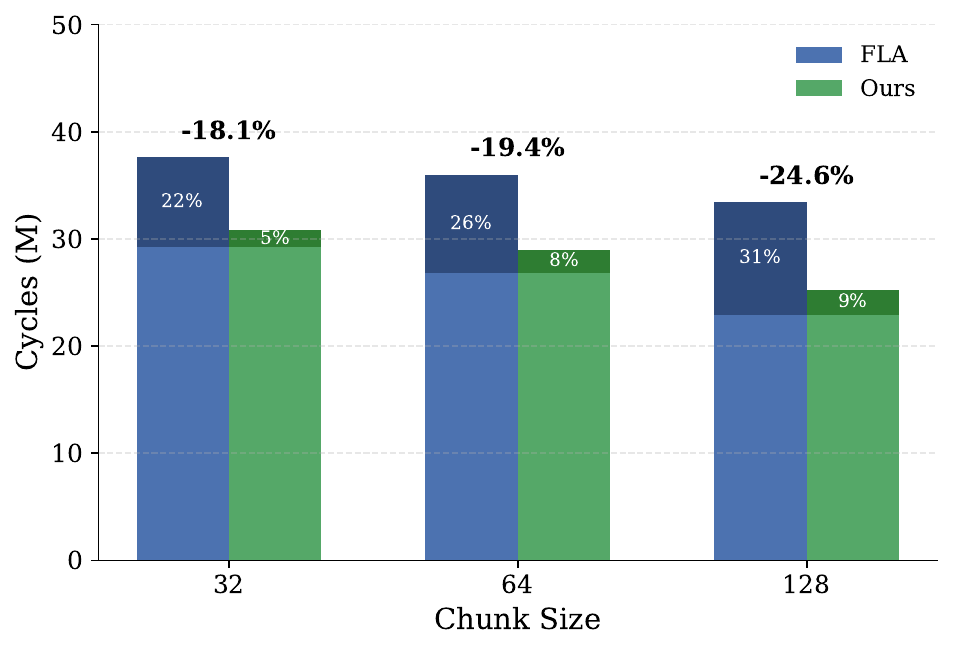}}
    \caption{
      Cycle breakdown across chunk sizes at fixed sequence length ($L=128$) on a GatedDeltaNet layer from Qwen3.5-4B. The lighter segment denotes base computation, while the darker segment highlights matrix-inverse overhead.
    }
    \label{fig:matrix_inverse_ratio}
  \end{center}
  \vskip -0.4 in
\end{figure}

To address this bottleneck, \citet{DBLP:conf/icml/YangWSPK24} reformulate token-level recurrent updates in Gated Linear Attention into chunk-level recurrences, enabling parallel computation within each chunk while preserving recurrent semantics across chunks. When applied to GatedDeltaNet, this strategy removes the outer token-by-token dependency and exposes locally parallel low-rank linear solves. However, it also introduces a new inner bottleneck: triangular matrix inversion.
As shown in Figure~\ref{fig:matrix_inverse_ratio}, matrix inversion is not a negligible overhead. On a GatedDeltaNet layer from Qwen3.5-4B with fixed sequence length ($L=128$), it accounts for 22.3\%--31.4\% of the total cycle cost across chunk sizes, and its relative contribution increases as the chunk size grows.
This overhead stems from the forward-substitution-based inversion procedure, whose limited parallelism and vector-heavy operations are poorly matched to NPU execution, leaving matrix processing units underutilized.

In GatedDeltaNet, the matrix inversion in fragments takes the form $T=(I-A)^{-1}, A \in \mathbb{R}^{k\times k}$. Since $A$ is a strict lower-triangular matrix, its inverse can be efficiently approximated by using the Neumann series expansion~\cite{golub2013matrix}. However, directly expanding the Neumann series for larger chunks can introduce rapidly growing intermediate values, increasing the risk of numerical instability under finite-precision arithmetic.
Previous work~\cite{DBLP:conf/icml/YangWSPK24,gdn_tri_inverse} adopts block-wise matrix inversion, decomposing a large matrix inversion into multiple smaller ones. While this improves numerical behavior, it also restricts the matrix size of each computation block, reducing available parallelism and limiting hardware utilization. These observations raise a central question: \textit{how can we design a MatMul-based algorithm for fast inversion of large matrices to accelerate the overall computation?}

We answer this question by reformulating such triangular inversion as an \emph{approximate-but-sufficient} MatMul-only computation.
Our key observation is that exact inversion is unnecessary: although high-order Neumann terms may produce large intermediate values, their impact mainly lies on deeper sub-diagonals, while the inverse energy is concentrated near the main diagonal.
Thus, a low-order Neumann approximation with structured masking and parallel residual correction preserves accuracy while mapping the dominant computation to matrix multiplications.

Our main contributions are summarized as follows:
\begin{itemize}
%  \item \textbf{Approximate-but-sufficient MatMul-only inversion.}
%   We reformulate strictly lower-triangular matrix inversion as a low-order Neumann approximation with diagonal masking and parallel residual correction, eliminating sequential dependencies while preserving accuracy under FP16.

% \item \textbf{Hardware-efficient execution.}
%   The proposed formulation exposes explicit control over truncation order and residual depth, achieving up to $5\times$ kernel speedup on NPUs without accuracy degradation.
\item \textbf{Approximate-but-sufficient MatMul-only inversion.}
  We replace sequential triangular inversion in chunk-wise GatedDeltaNet with a low-order Neumann approximation, structured diagonal masking, and parallel residual correction, preserving accuracy under floating-point and integer quantization while mapping the computation to MatMul kernels.

  \item \textbf{Hardware-efficient NPU execution.}
  The proposed formulation removes triangular-solve dependencies, exposes dense matrix multiplications with controllable truncation and correction depth, and achieves up to $5\times$ kernel speedup on NPUs without accuracy degradation.
\end{itemize}

% To answer this question, we reformulate matrix inversion as an \emph{approximate-but-sufficient, MatMul-only computation}. Our main contributions are:
% \begin{itemize}
%   \item \textbf{Approximate-but-sufficient MatMul-only inversion.}
%   We reformulate strictly lower-triangular matrix inversion as a low-order Neumann approximation with diagonal masking and parallel residual correction, eliminating sequential dependencies while preserving accuracy under FP and INT quantization.

%   \item \textbf{Hardware-efficient and quantization-robust design.}
%   The proposed formulation exposes explicit control over truncation order and residual depth, enabling near-optimal latency–accuracy trade-offs and achieving up to $5\times$ kernel speedup on NPUs without accuracy degradation.
% \end{itemize}

\section{Background}
\subsection{Matrix inversion in Gated DeltaNet}
Let $A \in \mathbb{R}^{k \times k}$ be a strictly lower-triangular matrix, i.e., $T_{ij} = 0$ for $j \ge i$. 
We consider the inversion of a matrix of the form 
\begin{equation}
T = (I - A)^{-1} .
\label{eq:I_minus_L}
\end{equation}

In Gated DeltaNet~\cite{DBLP:conf/iclr/YangKH25}, the matrix $A$ is constructed from the key representations via a scaled inner product of the form $A = K K^{\top}$. Since an $\ell_2$-normalization layer is applied after the key projection, which bounds the magnitude of the entries in $\rho(A) < 1$. 

\subsection{Forward Substitution}
In previous methods, the matrix inverse is computed via a forward-substitution procedure:
\begin{equation}
T_{i,0:i} \leftarrow T_{i,0:i} + T_{i,0:i}\, T_{0:i,0:i},
\qquad i = 1,\dots,k-1,
\label{eq:forward_sub}
\end{equation}
followed by an identity augmentation
\begin{equation}
T \leftarrow T + I .
\end{equation}

As shown in Eq.~\eqref{eq:forward_sub}, forward substitution requires $k-1$ sequential iterations for a $k\times k$ matrix, where each step depends on all previous results. These loop-carried dependencies hinder efficient parallelization and limit scalability on modern hardware. Flash Linear Attention~\cite{yang2024fla} addresses this by decomposing a large matrix inverse into multiple smaller inverses computed jointly; details are provided in Appendix~\ref{app:option2}.

\subsection{Neumann Series for Strictly Lower-Triangular Matrix Inversion}
When the spectral radius of $A$ satisfies $\rho(A) < 1$, the inverse can be expressed by the Neumann series as
\begin{equation}
(I - A)^{-1} = \sum_{n=0}^{\infty} A^n = I + A + A^2 + \cdots.
\end{equation}

Moreover, since $A$ is strictly lower triangular, it is nilpotent and satisfies
\begin{equation}
A^k = 0.
\end{equation}
Therefore, the above series terminates after finitely many terms, yielding the exact finite expansion
\begin{equation}
(I - A)^{-1} = \sum_{n=0}^{k-1} A^n.
\end{equation}

This finite-series form underpins matrix-inverse approximation in chunk-wise Gated DeltaNet and motivates Neumann-series implementations.

\begin{figure*}[t]
  % \vskip -0.2in
  \begin{center}
    \centerline{\includegraphics[width=\textwidth]{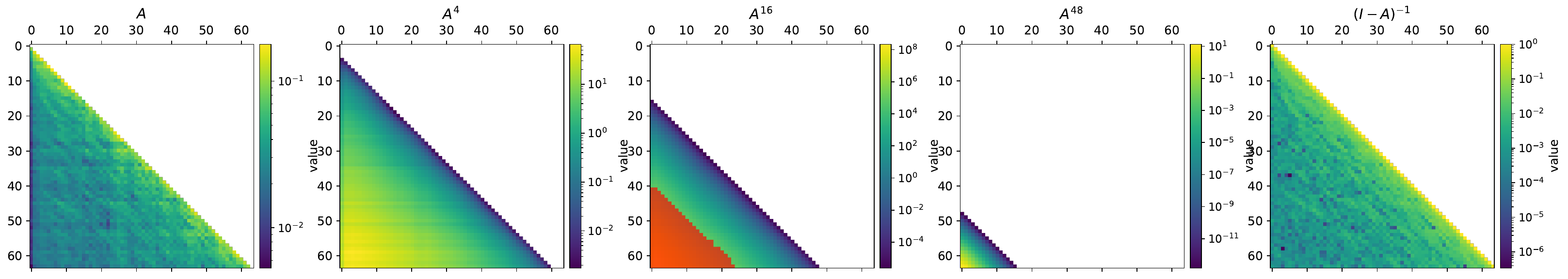}}
    \caption{
      Distribution of $A^n$ over 100 samples. Values exceeding the FP16 limit (65,504) are highlighted in red; 2 samples exhibit overflow, indicating heavy-tailed growth in higher-order terms.
    }
    \label{figure_A_dist}
    \vskip -0.4in
  \end{center}
\end{figure*}

\section{Method}

\subsection{Low-Order Truncation for Efficient Inversion}

% For a strictly lower-triangular matrix $A \in \mathbb{R}^{k \times k}$, the inverse $(I-A)^{-1}$ admits the exact finite expansion with $\sum_{n=0}^{k-1} A^n$. However, computing all $k$ terms is inefficient for large chunk sizes, as it requires a long chain of MatMul operations. 
Although the inverse $(I-A)^{-1}$ admits the exact finite expansion with $\sum_{n=0}^{k-1} A^n$, computing all $k$ terms is inefficient for large chunk size. 
A full expansion is often unnecessary in practice.  Because the inverse energy is concentrated near the main diagonal.

\begin{lemma}[Diagonal Localization of Neumann Series for strictly lower-triangular matrix]
\label{lemma:diagonal}
Let $A \in \mathbb{R}^{k \times k}$ be strictly lower triangular, i.e.,
$A_{ij}=0$ for $i \le j$.
Then for any $n \ge 0$,
\[
(A^{n})_{ij} = 0 \;\Rightarrow\; i-j < n .
\]
\end{lemma}

Based on Lemma~\ref{lemma:diagonal}, each power $A^n$ only contributes to the $n$-th sub-diagonal and below, a low-order truncation already captures most of the useful structure of the inverse. Therefore, the Neumann series can be approximated by its first $N$ terms:
\begin{equation}
T^{(0)} = \sum_{n=0}^{N} A^n, \qquad N \ll k,
\end{equation}
as the initial approximation. The truncation error is
\begin{equation}
E_N = \sum_{n=N+1}^{k-1} A^n,
\end{equation}
which decays rapidly as $N$ increases when $\|A\|<1$.

\begin{lemma}[Truncation Error of Finite Neumann Expansion]
\label{lem:truncation_error}
Let $A \in \mathbb{R}^{k \times k}$ be a strictly lower-triangular matrix, 
for any truncation order $N < k-1$ and for any sub-multiplicative matrix norm $\|\cdot\|$ and $\|A\| < 1$, then
\begin{equation}
\|E_N\|
\le
\frac{\|A\|^{N+1}}{1-\|A\|} .
\end{equation}
\end{lemma}

This lemma and empirical study, in Appendix~\ref{app:Low-Order Truncation}, shows that a much lower truncation order($N \approx 30$) already captures most of the inverse structure, making full expansion unnecessary in practice, although $(I-A)^{-1}$ admits an exact finite expansion for strictly lower-triangular $A$.

\subsection{Numerical Overflow in Truncated Neumann Series}
\label{Sec:distribution_analysis}
We analyze the distribution of individual entries in the Neumann series by examining average across 100 samples. Fig.~\ref{figure_A_dist} illustrates the matrices at different truncation orders. Based on these empirical observations, we derive the following two findings.

\textbf{Although the entries of $A$ are bounded in $[0,1]$, the magnitudes of higher-order terms in the Neumann series grow rapidly with increasing powers.} In particular, as $n$ increases, the distribution of $A^n$ exhibits heavier-tailed distributions, with many entries exceeding the FP16 range.
This effect, particularly visible in higher-order terms, shows that overflow can occur even for well-conditioned matrices. Thus, forward-substitution FP16 Neumann implementations are limited by element-wise growth rather than spectral properties alone.

\begin{lemma}[Exact Entrywise Growth of Powers of Strictly Lower-Triangular Matrices]
\label{lemma:upper_bound}
Let $A \in \mathbb{R}^{k \times k}$ be a strictly lower-triangular matrix satisfying $\rho(A) < 1$
% \[
% A_{ij} = 0 \quad \text{for } i \le j,
% \qquad
% 0 \le A_{ij} \le 1 \quad \text{for } i > j.
% \]

For indices $i > j$, define the diagonal distance $d \triangleq i - j$.
Then, for any integer $n \ge 1$,
\[
(A^n)_{ij} \;\le\; \binom{d-1}{n-1}.
\]
\end{lemma}

Based on Lemma~\ref{lemma:upper_bound}, we can directly derive the numerical stability bounds for the Neumann-series-based inversion. 
For a $64 \times 64$ matrix, the worst-case upper bound is $6$ under FP16 precision and $32$ under FP32 precision.

\textbf{The data distribution of $A^n$ and $(I-A)^{-1}$ exhibits a strong block-wise structure.} The dominant contributions of the inverse matrix $(I - A)^{-1}$ are primarily concentrated along the diagonal and near-diagonal entries. While values that exceed the representable precision range are primarily concentrated in the lower-left (strictly lower-triangular) region of the matrix.

By Lemma~\ref{lemma:diagonal}, the $n$-th Neumann term contributes only to the $n$-th sub-diagonal and below, leading to strong diagonal concentration in $(I-A)^{-1}$. Consequently, truncated Neumann series initialization can efficiently obtain an accurate approximation. 
As analysis in Appendix~\ref{app:Low-Order Truncation}, the Neumann series in practice typically requires more than 20 terms to achieve over 0.99 power ratio.
Therefore, maintaining a small truncation order for numerical stability leads to noticeable accuracy degradation. To address this trade-off, we introduce a residual correction stage following the truncated Neumann approximation, as described in Sec.~\ref{Sec:Residual_correction}.

% As shown in Fig.~\ref{fig:diagonal_power}, most cases converge rapidly, although some require more than 20 steps to exceed 99\% accumulated energy. Numerical overflow depends on both the chunk size and truncation order; therefore, maintaining a small truncation order is critical for ensuring numerical stability.

\subsection{Effective Diagonal Masking}

Low-order truncation preserves most of the meaningful values near the diagonal, which exactly match the corresponding entries of $(I - A)^{-1}$.
However, the truncation error increases with the magnitude of the Neumann-series terms, by Lemma~\ref{lem:truncation_error} .
Such artifacts can adversely affect subsequent computations based on the approximate inverse and are particularly problematic under INT quantization, where min--max calibration is sensitive to outliers.

% To mitigate these effects, we apply an effective diagonal mask that suppresses off-diagonal entries corresponding to already-resolved components of the approximation.
Based on the observed block-wise structure of $A^n$ and $(I-A)^{-1}$, we apply a diagonal mask to separate the already well-resolved components from those with large truncation errors.
The diagonal mask $M^{(K)}_{ij}$ is defined as
\begin{equation}
    M^{(K)}_{ij} =
\begin{cases}
1, & i-j \le K, \\
0, & otherwise.
\end{cases}
\end{equation}
Where the $K$ is set equal to the Neumann series truncation order $N$.
This masking strategy effectively removes extraneous large-magnitude values, improving numerical stability and making the approximation quantization-friendly.

\subsection{Residual Correction}
\label{Sec:Residual_correction}
After applying a low--order Neumann series expansion together with the diagonal mask, we obtain an initial approximation of $(I-A)^{-1}$, denoted as $T^{(0)}$:
\begin{equation}
    T^{(0)} = M^{(N)}  \sum_{n=0}^{N} A^{n}.
\end{equation}

Although diagonal masking suppresses large outliers, non‑negligible approximation errors remain in the lower‑triangular region, necessitating further refinement. Following a matrix‑calculation–based design rule, we adopt a residual correction scheme that iteratively improves the approximation without increasing the Neumann order.

Given the current estimate $T^{(m)}$, the residual and update rule are defined as
\begin{equation}
    \begin{aligned}
        R^{(m)} &= I - (I - A) T^{(m)}, \\
        T^{(m+1)} &= T^{(m)} + T^{(m)} R^{(m)} .
    \end{aligned}
\end{equation}
This formulation applies a first‑order multiplicative correction to compensate for the remaining inversion error. Empirically, our ablation study shows that only 2--3 iterations are sufficient for convergence.

However, the iterative process is inherently sequential: $R^{(m)}$ depends on $T^{(m)}$, and  $T^{(m+1)}$ depends on its predecessor. This step-wise dependency limits parallelism and reduces practical efficiency.
To address this issue, we reformulate residual correction as an accumulation of matrix powers, eliminating explicit iteration. Specifically, we define
\begin{equation}
    \begin{aligned}
        E &= I - (I - A) T^{(0)}, \\
        T &\approx T^{(0)} \sum_{s=0}^{S} E^{s}.
    \end{aligned}
\end{equation}
With $S = 2^{M}$ (typically $S=4$--$8$), this approach enables MatMul-only execution, improving parallelism and hardware efficiency while maintaining correction accuracy.

\begin{algorithm}[t]
\caption{Masked Neumann Initialization with Iterative Residual Correction}
\label{alg:iter_residual}
\begin{algorithmic}[1]
\REQUIRE Strictly lower-triangular $A\in\mathbb{R}^{k\times k}$; Neumann order $N$; residual iterations $m_{\max}$
\ENSURE Approximate inverse $T \approx (I-A)^{-1}$

\STATE Construct mask $M^{(N)}$ where $M^{(N)}_{ij}=1$ if $i-j\le N$, else $0$
\STATE $T^{(0)} \leftarrow I$, \;\; $P \leftarrow I$
\FOR{$n=1$ to $N$}
    \STATE $P \leftarrow P A$
    \STATE $T^{(0)} \leftarrow T^{(0)} + P$
\ENDFOR
\STATE $T^{(0)} \leftarrow M^{(N)} \odot T^{(0)}$
\STATE $E \leftarrow I - (I - A) T^{(0)}$
\STATE $T \leftarrow I$, $P \leftarrow I$
\FOR{$s=0$ to $s_{\max}-1$}
    \STATE $P \leftarrow P E$
    \STATE $T \leftarrow T + P$
\ENDFOR
\STATE \textbf{return} $T \leftarrow TT^{(0)}$
\end{algorithmic}
\end{algorithm}

\subsection{Good-Enough under Low-Bit Quantization}
The \emph{good-enough} principle seeks the minimal Neumann truncation order that preserves numerical stability while introducing Non (or negligible) accuracy loss. 
This rationale is particularly strong under low-bit quantization (e.g., INT8), where the iterative Neumann structure repeatedly applies quantization to the same values, causing quantization noise to dominate the overall error. 
In this regime, the additional approximation error from low-order truncation becomes insignificant, making reduced computation a preferable trade-off for improved hardware efficiency and lower inference latency. 

As shown in Fig.~\ref{fig:quant_range}, increasing the truncation order from $n=3$ to $n=4$ dramatically enlarges the dynamic range: the maximum absolute value expands from 978 to 16216. 
This expansion is highly non-uniform: a few entries grow explosively while most remain tightly concentrated, yielding a pronounced heavy tail. Under uniform quantization, the scaling factors are therefore dictated by these rare extreme values rather than the representative bulk. Consequently, most entries are compressed into a severely limited effective resolution, which exacerbates rounding errors and leads to a significant degradation in numerical fidelity under INT8/INT16 precision constraints.

As a result, higher-order truncation introduces substantial quantization instability while providing only marginal approximation gains. We therefore limit the truncation order ($N \leq 3$) and compensate with residual correction, achieving an efficient matrix inverse that balances numerical stability, accuracy, and efficiency under low-precision settings.

\begin{figure}[thb]
  \begin{center}
    \centerline{\includegraphics[width=\columnwidth]{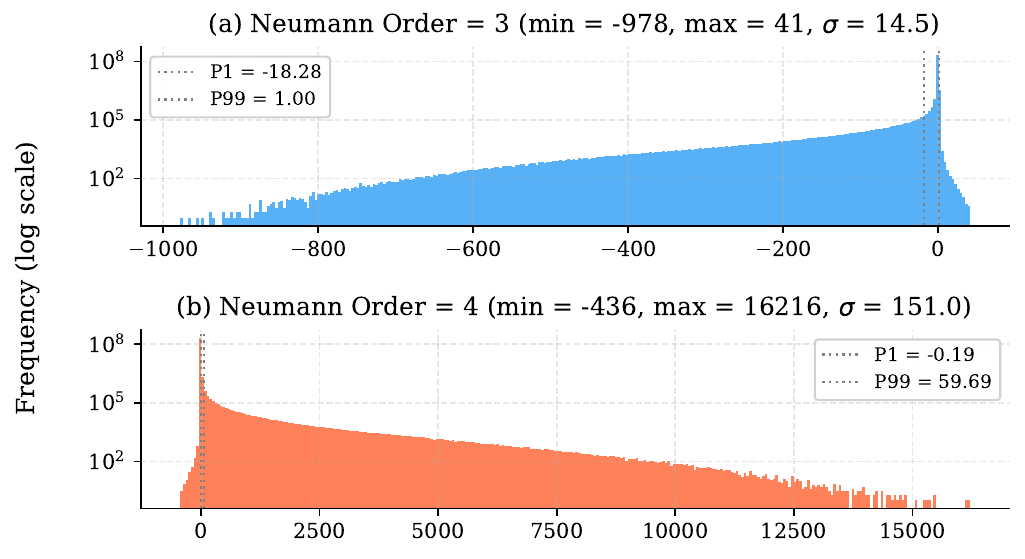}}
    \caption{
     Activation distribution under Neumann truncation for order=3 and 4 for a 64$\times$64 matrix Neumann Series.
    }
    \label{fig:quant_range}
  \end{center}
  \vskip -0.4 in
\end{figure}

\section{Experiments}
\subsection{Experiments Setting}
We select the Qwen3Next~\cite{qwen3technicalreport} and Qwen3.5~\cite{qwen3.5} model families, to complete full accuracy and on-target latency study. 
Unless otherwise stated, all experiments are conducted with a chunk size $k = 64$, a Neumann series order $N = 3$, and $S = 8$ residual correction steps. 

\textbf{Single‑kernel accuracy analysis.}
We conduct experiments using 100 samples from the WikiText‑v2~\cite{DBLP:conf/iclr/MerityX0S17} training set with a sequence length of 4K, based on Qwen3-Next-80B-A3B-Instruct model \footnote{https://huggingface.co/Qwen/Qwen3-Next-80B-A3B-Instruct}. 

\textbf{End‑to‑end accuracy analysis.}
Based on multiple model scales of the Qwen3.5 family, the evaluation includes perplexity (PPL) measured on the WikiText‑v2 validation, and also assess downstream task accuracy on benchmarks such as MMLU~\cite{DBLP:conf/iclr/HendrycksBBZMSS21}, CSR~\cite{DBLP:conf/naacl/ClarkLCK0T19,DBLP:journals/corr/abs-1803-05457,DBLP:conf/aaai/BiskZLGC20,DBLP:conf/acl/ZellersHBFC19} and  multi-modality RealWorldQA~\cite{realworldqa}\footnote{Thinking mode is disabled for efficient evaluation, resulting in a slight drop in RealWorldQA accuracy.}. 
In the quantization, we apply W4A16 to the LLM decoder and W8A16 to the visual decoder. Detailed quantization configurations are provided in the Appendix~\ref{app:quant_setting}.

\textbf{On-target performance analysis.}
We profile the runtime cost of a single matrix inversion operation and a single GatedDeltaNet layer on the Snapdragon 8 Elite Gen 5 platform. We compare our approach against chunk‑wise parallel implementation from Flash Linear Attention.

\subsection{Accuracy Experiments}

\textbf{Single kernel experiment.}
During inference, we collect the inputs and outputs of each matrix inversion operation from all layers and quantify the approximation error of the proposed matrix inversion method using the signal‑to‑noise ratio (SNR). 
As shown in Table~\ref{sample-table}, FP16 and INT16 introduce only marginal MSE increases compared to FP32 while maintaining high SNR. INT16 remains comparable to FP16, demonstrating stable single‑kernel behavior under low‑precision execution. 
\begin{table}[h]
  \caption{Single Kernel accuracy}
  \vskip -0.1 in
  \label{sample-table}
  \begin{center}
    \begin{small}
      \begin{sc}
        \begin{tabular}{lcccc}
          \toprule
          Metric & FP32 & FP16 & INT16 \\
          \midrule
          SNR & 70.02 & 66.78 &  67.16 \\
          \bottomrule
        \end{tabular}
      \end{sc}
    \end{small}
  \end{center}
  \vskip -0.2 in
\end{table}

\textbf{End-to-End experiments.}
Table~\ref{tab:qwen35_e2e} reports end‑to‑end accuracy across Qwen3.5 models from 0.8B to 9B parameters. Across all scales, our method matches the FLA baseline in perplexity, indicating no degradation in language modeling quality. On downstream tasks, including MMLU, CSR, and RealWorldQA, our approach achieves comparable performance with only minor fluctuations (within ±0.3), demonstrating that the proposed modification preserves task accuracy across model sizes. These results confirm that our method introduces no observable accuracy regression while enabling the targeted efficiency optimizations. 
\begin{table}[htb]
  \caption{End-to-end accuracy across Qwen3.5 model sizes.}
  \vskip -0.1 in
  \label{tab:qwen35_e2e}
  \begin{center}
    \begin{small}
      \begin{sc}
        \begin{tabular}{llccccc}
          \toprule
          Model & Method & PPL & MMLU & CSR & RWQA \\
          \midrule
          \multirow{2}{*}{0.8B}
            & FLA &  15.74 & 50.57 & 51.11 & 62.35 \\
            & Ours  &  15.74 & 50.61 & 51.17 & 61.70 \\
          \midrule
          \multirow{2}{*}{2B}
            & FLA & 11.07 & 57.66 & 56.99 & 65.23 \\
            & Ours  & 11.07 & 57.72 & 56.93 & 65.62 \\
          \midrule
          \multirow{2}{*}{4B}
            & FLA & 8.89 & 70.10 & 65.48 & 74.77 \\
            & Ours  & 8.89 & 70.11 & 65.45 & 74.64 \\
          \midrule
          \multirow{2}{*}{9B}
            & FLA & 8.21 & 70.26 & 67.28 & 74.38 \\
            & Ours  & 8.21 & 70.23 & 67.30 & 74.12 \\
          \bottomrule
        \end{tabular}
      \end{sc}
    \end{small}
  \end{center}
\vskip -0.2 in
\end{table}

\textbf{Quantization experiments.}
Table~\ref{tab:qwen35_fp_int16} compares end‑to‑end accuracy under W4A16 quantization across Qwen3.5 model sizes. While FLA quantization leads to noticeable degradation compared with full‑precision baselines, our method preserves performance more effectively, particularly on RealWorldQA, with only marginal differences in PPL, MMLU, and CSR. Across both 0.8B and 4B models, these results demonstrate that the proposed approach better maintains task accuracy under aggressive low‑precision quantization.
Besides, we also explore the low bits setting (INT8) in matrix inverse, and its accuracy is almost same to INT16, shown in Appendix~\ref{sec:w4a16_qwen}
\begin{table}[!thb]
    \setlength{\tabcolsep}{4pt}
    \renewcommand{\arraystretch}{0.95}
  \caption{quantization accuracy comparison across Qwen3.5 model}
  \vskip -0.1 in
  \label{tab:qwen35_fp_int16}
  \begin{center}
    \begin{small}
      \begin{sc}
        \begin{tabular}{llcccc}
          \toprule
          Model & Method & PPL & MMLU & CSR & RWQA \\
          \midrule
          \multirow{3}{*}{0.8B} & FLA-FP & 15.74 & 50.57 & 51.11 & 62.35 \\
                       & FLA-W4A16 & 17.54 & 48.26 & 49.03 & 56.08 \\
                       & Ours-W4A16 & 17.55 & 48.27 & 49.21 & 60.39 \\
          \midrule
          \multirow{3}{*}{4B}   & FLA-FP &  8.89 & 70.10 & 65.48 & 74.77 \\
                       & FLA-W4A16 & 9.66 & 73.34 & 65.05 & 73.33 \\
                       & Ours-W4A16 & 9.67 & 73.29 & 65.03 & 72.81 \\
          \bottomrule
        \end{tabular}
      \end{sc}
    \end{small}
  \end{center}
  \vskip -0.3in
\end{table}

\subsection{On-target performance experiments}

For a chunk-size of $k=32$, we found that $N=3$ and $S=4$ are sufficient to achieve accurate results. To test the end-to-end performance on different CL and chunk size, we included an ablation study in Appendix~\ref{app:cl_ablation}.

\textbf{Single kernel performance analysis.}
As shown in Figure~\ref{fig:single_kernel}, our method reduces single matrix-inverse cycles across chunk sizes, achieving 5.2×, 4.2×, and 4.6× improvements at 32, 64, and 128, respectively.
Chunk size 32 uses 7 matmuls, while 64 requires 11. Despite $4\times$ smaller matmuls at size 32, latency does not scale proportionally, indicating dominance of non-matmul overheads. 
At chunk size 128, splitting and non-contiguous write-back of $64 \times 64$ submatrices introduces memory overhead, which is observed in both our method and FLA and dampens scaling efficiency.

\begin{figure}[b]
  \begin{center}
  \vskip -0.2 in
    \centerline{\includegraphics[width=\columnwidth]{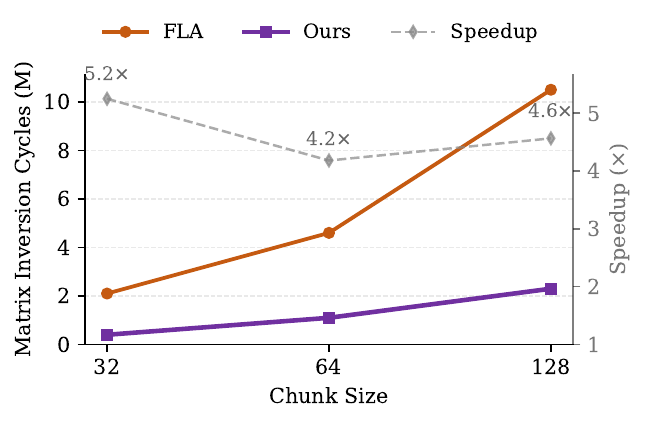}}
    \caption{
     Plot of single kernel performance across different chunk wise. Here, H=32, $D_k$=128.
    }
    \label{fig:single_kernel}
  \end{center}
  \vskip -0.4 in
\end{figure}

\textbf{GatedDeltaNet performance analysis.}
% Figure~\ref{fig:matrix_inverse_ratio} compares the cycle breakdown between FLA and our method under different chunk sizes. While the computation excluding matrix inversion is comparable, FLA incurs a substantial overhead from matrix inversion, accounting for 22.3\% and 25.6\% of total cycles at chunk sizes 32 and 64, respectively. In contrast, our approach significantly reduces the matrix‑inverse cost to 5.2\% and 7.6\%, leading to overall cycle reductions of 18.1\% and 19.4\%. These results demonstrate that our design effectively mitigates the matrix inversion bottleneck, yielding consistent end‑to‑end speedups across chunk sizes.
Figure~\ref{fig:matrix_inverse_ratio} compares the cycle breakdown between FLA and our method across different chunk sizes. While the computation excluding matrix inversion remains comparable, FLA incurs substantial overhead from matrix inversion, accounting for 22.3\%, 25.6\%, and $\sim$31\% of total cycles at chunk sizes 32, 64, and 128, respectively. In contrast, our approach reduces the matrix-inverse cost to 5.2\%, 7.6\%, and $\sim$9\%, yielding overall cycle reductions of 18.1\%, 19.4\%, and 24.6\%. Combining the single-kernel results, we observe that as non-matmul operations become faster, the relative benefit of our matrix inversion method becomes more pronounced.

\subsection{Ablation Study}

\textbf{Effect of each module.}
% Table~\ref{tab:ablation_snr} ablates key components of the Neumann-series approximation. While FP64 with $N=64$ achieves near-perfect accuracy (178.42/96.46 dB mean/worst SNR), reducing to FP16 causes severe instability, with worst-case SNR dropping to $-17.94$ dB. Further truncation ($N=3$) introduces large errors and divergence ($-51.11$ dB worst-case). Residual correction ($S=8$) recovers mean accuracy (80.35 dB) but fails to stabilize worst-case behavior. In contrast, the diagonal mask significantly improves robustness, increasing mean/worst SNR to 86.91/47.98 dB. These results show that low-order approximation alone is insufficient, and that diagonal masking is essential for controlling worst-case errors under low-precision settings.
Table~\ref{tab:ablation_snr} ablates the key components of the Neumann-series approximation. In high precision, FP64 with $N=64$ achieves near-perfect accuracy (178.42/96.46 dB mean/worst SNR). However, reducing precision to FP16 introduces severe instability, with the worst-case SNR dropping to $-17.94$ dB. Further truncation ($N=3$) leads to substantial errors and divergence ($-51.11$ dB worst-case). Residual correction ($S=8$) partially recovers mean accuracy (80.35 dB) but fails to stabilize worst-case behavior. In contrast, applying the diagonal mask markedly improves robustness, increasing mean/worst SNR to 86.91/47.98 dB. These results indicate that low-order approximation alone is insufficient, and that diagonal masking is essential for controlling worst-case errors (overflowing cases) under low-precision settings.
\begin{table}[h]
  \vskip -0.05 in
  \caption{Ablation study on Neumann-series approximation components.}
    \vskip -0.1 in
  \label{tab:ablation_snr}
  \begin{center}
    \begin{small}
      \begin{sc}
        \begin{tabular}{lcc}
          \toprule
          Method & ${SNR}_{mean}$ & ${SNR}_{worst}$ \\
          \midrule
          FP64 $N=64$ & 178.42 & 96.46 \\
          $\Rightarrow$FP16 & 79.53 & -17.94 \\
          $\Rightarrow N=3$ & 42.13 & -51.11 \\
          +$S=8$ & 80.35 & -4.24 \\
          +Diagonal mask & 86.91 & 47.98 \\
          \bottomrule
        \end{tabular}
      \end{sc}
    \end{small}
  \end{center}
  \vskip -0.2in
\end{table}
% As shown in Table~\ref{tab:ablation_snr}, directly reducing precision to FP16 and truncating the Neumann series ($N=3$) leads to severe numerical instability. Introducing residual correction ($S=8$) substantially recovers both mean and worst‑case SNR, while the diagonal mask further suppresses error accumulation along fixed diagonals. Together, these components are critical for stabilizing short‑horizon, low‑precision Neumann approximations.

\textbf{Order of Neumann series.}
Table~\ref{tab:order_step_ablation} studies the interaction between Neumann truncation order $N$ and residual correction steps $S$. Small $N$ with insufficient correction leads to severe numerical instability, resulting in divergence (NaN) or extremely large error. Increasing $S$ consistently stabilizes the approximation, with performance saturating around $S=6\!-\!8$ for $N=3$ and $N=4$. Higher $N$ requires fewer correction steps but worse quantization perforamnce once stability is achieved. These results suggest that INT16 quantization can't cover the dynamic range of $N=4$. Based on these observation, we finally choose $N=3,S=8$ for $64\times64$ matrix inversion.

\begin{table}[h]
\vskip -0.05 in
  \caption{Effect of Neumann series order \(n\) and residual steps \(s\) on numerical stability. Baseline W8A16 is 8.98.}
  \label{tab:order_step_ablation}
  \vskip -0.1 in
  \begin{center}
    \begin{small}
      \begin{sc}
        \begin{tabular}{lcccc}
          \toprule
          & $N=3$ & $N=4$ & $N=5$ & $N=6$ \\
          \midrule
          $S=1$ & NaN   & NaN   & NaN      & 469727.0 \\
          $S=2$ & NaN   & NaN   & 62789.47 & 511.25   \\
          $S=3$ & NaN   & NaN   & 1390.22  & 92.47    \\
          $S=4$ & NaN   & 169.30 & 61.31   & 144.95  \\
          $S=5$ & NaN   & 9.75   & 81.16   & 102.24  \\
          $S=6$ & 23.56 & 9.54   & 63.52   & 135.83  \\
          $S=7$ & 8.99  & 9.54   & 81.64   & 119.13  \\
          $S=8$ & \textbf{8.98} & 9.54 & 66.03 & 117.79 \\
          $S=9$ & \textbf{8.98} & 9.54 & 67.32 & 128.32 \\
          \bottomrule
        \end{tabular}
      \end{sc}
    \end{small}
  \end{center}
  \vskip -0.2 in
\end{table}

\section{Conclusion}

% We identify matrix inversion in chunk-wise GatedDeltaNet as a key bottleneck for long-context linear attention, particularly under low-precision inference. We propose a structure-aware, MatMul-based inversion algorithm combining truncated Neumann expansion, diagonal masking, and parallel residual correction to eliminate sequential dependencies and stabilize computation. Overall, our approach delivers up to 5× kernel-level speedup and 20\% lower decode-layer overhead, while remaining numerically stable under INT16 quantization. These results demonstrate an efficient and hardware-friendly solution for scalable linear attention on NPUs.
We identify matrix inversion in chunk-wise GatedDeltaNet as a critical bottleneck for long-context linear attention, particularly under low-precision inference. To address this, we propose a structure-aware, MatMul-based inversion algorithm that combines truncated Neumann expansion, diagonal masking, and parallel residual correction to eliminate sequential dependencies and improve numerical stability.
Our method achieves up to 5× kernel speedup and 20\% lower decode overhead, remaining robust under INT16 quantization and enabling efficient NPU deployment.

\section{Impact Statement}
This work improves the efficiency of large language models by enabling hardware-friendly, low-cost matrix inversion, facilitating deployment on resource-constrained devices such as mobile and edge platforms. By reducing compute and energy requirements, it can broaden access to long-context models and support more scalable, sustainable AI systems.
However, increased accessibility of efficient LLMs may also lower the barrier for misuse, including large-scale automated content generation or deployment in sensitive domains without sufficient safeguards. Additionally, low-precision approximation may affect reliability under certain conditions.
We emphasize the importance of careful evaluation and responsible deployment when applying such efficiency-oriented techniques.

% \section{Acknowledge}

% \bibliography{example_paper}

% \bibliographystyle{icml2026}

%%%%%%%%%%%%%%%%%%%%%%%%%%%%%%%%%%%%%%%%%%%%%%%%%%%%%%%%%%%%%%%%%%%%%%%%%%%%%%%
%%%%%%%%%%%%%%%%%%%%%%%%%%%%%%%%%%%%%%%%%%%%%%%%%%%%%%%%%%%%%%%%%%%%%%%%%%%%%%%
% APPENDIX
%%%%%%%%%%%%%%%%%%%%%%%%%%%%%%%%%%%%%%%%%%%%%%%%%%%%%%%%%%%%%%%%%%%%%%%%%%%%%%%
%%%%%%%%%%%%%%%%%%%%%%%%%%%%%%%%%%%%%%%%%%%%%%%%%%%%%%%%%%%%%%%%%%%%%%%%%%%%%%%
\newpage
\appendix
\onecolumn
%%%%%%%%%%%%%%%%%%%%%%%%%%%%%%%%%%%%%%%%%%%%%%%%%%%%%%%%%%%%%%%%%%%%%%%%%%%%%%%
%%%%%%%%%%%%%%%%%%%%%%%%%%%%%%%%%%%%%%%%%%%%%%%%%%%%%%%%%%%%%%%%%%%%%%%%%%%%%%%

\section{Related Works}
Standard attention~\cite{DBLP:conf/nips/VaswaniSPUJGKP17} has suffered from the quadratic time complexity to deal with the long context. Linear attention~\cite{DBLP:conf/aaai/BiskZLGC20} replaces the softmax operation with a positive feature map, enabling the attention computation to be reformulated as two associative matrix multiplications. This reformulation reduces the memory cost from sequence-length dependent to fixed-size, and lowers the computational complexity of similarity evaluation from $\mathcal{O}(T^2)$ to $\mathcal{O}(T)$.
Recent linear-attention-based works construct a \textit{Diagonal-Plus-Low-Rank} (DPLR) structure~\cite{S4,DBLP:conf/iclr/YangKH25,DBLP:journals/corr/abs-2510-26692,deltaformer,pathattention,rwkv7}, defined as $D-ab^T$.During computation, this structured matrix is typically diagonalized in the complex domain via joint diagonalization, enabling efficient evaluation of matrix powers or exponentials.
As such structure do calculation with recurrent, making the model hard to parallel. To solve this problem, Flash Linear Attention~\cite{DBLP:conf/icml/YangWSPK24} introduces chunk-wise parallelism recursive inversion. Following this transformation, the DPLR-based models can be further reformulated into a chunk-wise parallel computation framework. FlahsQLA~\cite{flashqla2025} applies reasonable operator fusion and performance optimization to the forward and backward passes of GDN~\cite{DBLP:conf/iclr/YangKH25} Chunked Prefill. However, these approaches do not address the bottleneck associated with matrix inversion. They are all use forward substitution to calculate the matrix inversion.
As the matrix size grows, block-wise matrix inversion is employed to facilitate parallel computation, as described in Appendix~\ref{app:option2}.
DeltaFormer~\cite{deltaformer} employs an algebraic factorization of the Neumann series to approximate matrix inversion, but overlooks the numerical stability issues under low-bit quantization. In contrast, GDN Tri-Inverse~\cite{gdn_tri_inverse} provides a more robust implementation by explicitly accounting for such numerical constraints. Specifically, it restricts the maximum size of a single matrix inversion to $16 \times 16$ and leverages block-wise inversion to enable parallel computation.

In this work, we achieve numerically stable matrix inversion for matrices up to $64 \times 64$ and further extend the computation to INT16 precision. Compared to prior approaches, our method significantly enlarges the feasible inversion size while maintaining numerical robustness under low-bit quantization. This enables more efficient large-scale matrix operations, reduces the reliance on fine-grained block partitioning, and improves hardware utilization by aligning with integer-friendly accelerators. Consequently, our approach provides a practical foundation for scalable and low-power deployment of linear-attention-based models on edge devices.

\section{Truncation Error of Finite Neumann Expansion}
\label{app:Low-Order Truncation}
\begin{proof}
Since $A$ is strictly lower triangular, it is nilpotent and satisfies $A^k=0$, hence
\begin{equation}
(I-A)^{-1} = \sum_{n=0}^{k-1} A^n.
\end{equation}
Therefore,
\begin{equation}
E_m = \sum_{n=m+1}^{k-1} A^n.
\end{equation}
Applying the triangle inequality and submultiplicativity yields
\begin{equation}
\|E_m\|
\le
\sum_{n=m+1}^{k-1} \|A^n\|
\le
\sum_{n=m+1}^{k-1} \|A\|^n.
\end{equation}
If $\|A\|<1$, the finite sum is further bounded by the tail of a convergent geometric series:
\begin{equation}
\sum_{n=m+1}^{k-1} \|A\|^n
\le
\sum_{n=m+1}^{\infty} \|A\|^n
=
\frac{\|A\|^{m+1}}{1-\|A\|}.
\end{equation}
This completes the proof.
\end{proof}

\paragraph{Empirical validation (chunk size $k=64$).}
We empirically validate the error with chunk size $k=64$.
We construct a strictly lower-triangular matrix $A\in\mathbb{R}^{64\times64}$ whose entries
decay with the distance to the main diagonal (to mimic diagonal-concentrated structure),
and rescale it to satisfy $\|A\|_2=0.55<1$.
We compute the exact inverse via the finite expansion
$T=\sum_{n=0}^{63}A^n$ and the truncated approximation
$T_m=\sum_{n=0}^{m}A^n$ for $m=0,1,\ldots,63$.
We report truncation errors in Frobenius and spectral norms,
$\|E_m\|_F$ and $\|E_m\|_2$, and compare them with the geometric upper bound
$\|A\|_2^{m+1}/(1-\|A\|_2)$ from Lemma~\ref{lem:truncation_error}.

\paragraph{Results.}
Figure~\ref{fig:truncation_k64} shows that truncation errors decrease rapidly with $m$
and enter a near-flat regime far before the full expansion order $k-1$.
Using a simple captured-structure proxy
$
r_m = 1-\|E_m\|_F/\|T\|_F,
$
we observe that $m=3$ already captures $\approx 95.5\%$ of the inverse structure,
$m=5$ captures $\approx 98.8\%$,
and $m=8$ captures $\approx 99.8\%$ in this $k=64$ study. 

\begin{figure}[t]
    \centering
    \includegraphics[width=\linewidth]{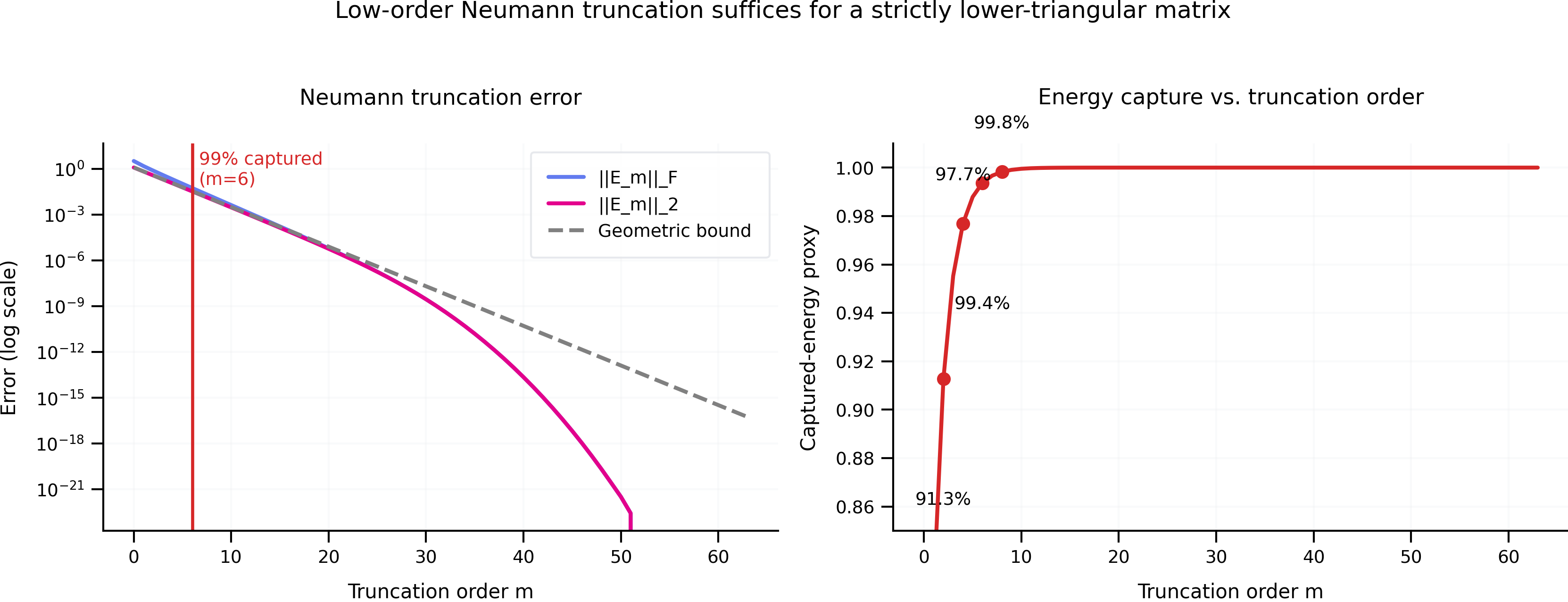}
    \caption{\textbf{Low-order truncation is sufficient for a strictly lower-triangular $64\times64$ example.}
    Left: truncation errors $\|E_m\|_F$ and $\|E_m\|_2$ versus truncation order $m$, together with the geometric-series upper bound
    $\|A\|_2^{m+1}/(1-\|A\|_2)$.
    Right: a captured-structure proxy $1-\|E_m\|_F/\|T\|_F$.
    Errors drop sharply within the first few orders, indicating that most inverse structure is captured by $m\ll k$, making full expansion unnecessary in practice.}
    \label{fig:truncation_k64}
    \vskip -0.2 in
\end{figure}

\begin{figure}[htb]
  \begin{center}
    \centerline{\includegraphics[width=0.5\linewidth]{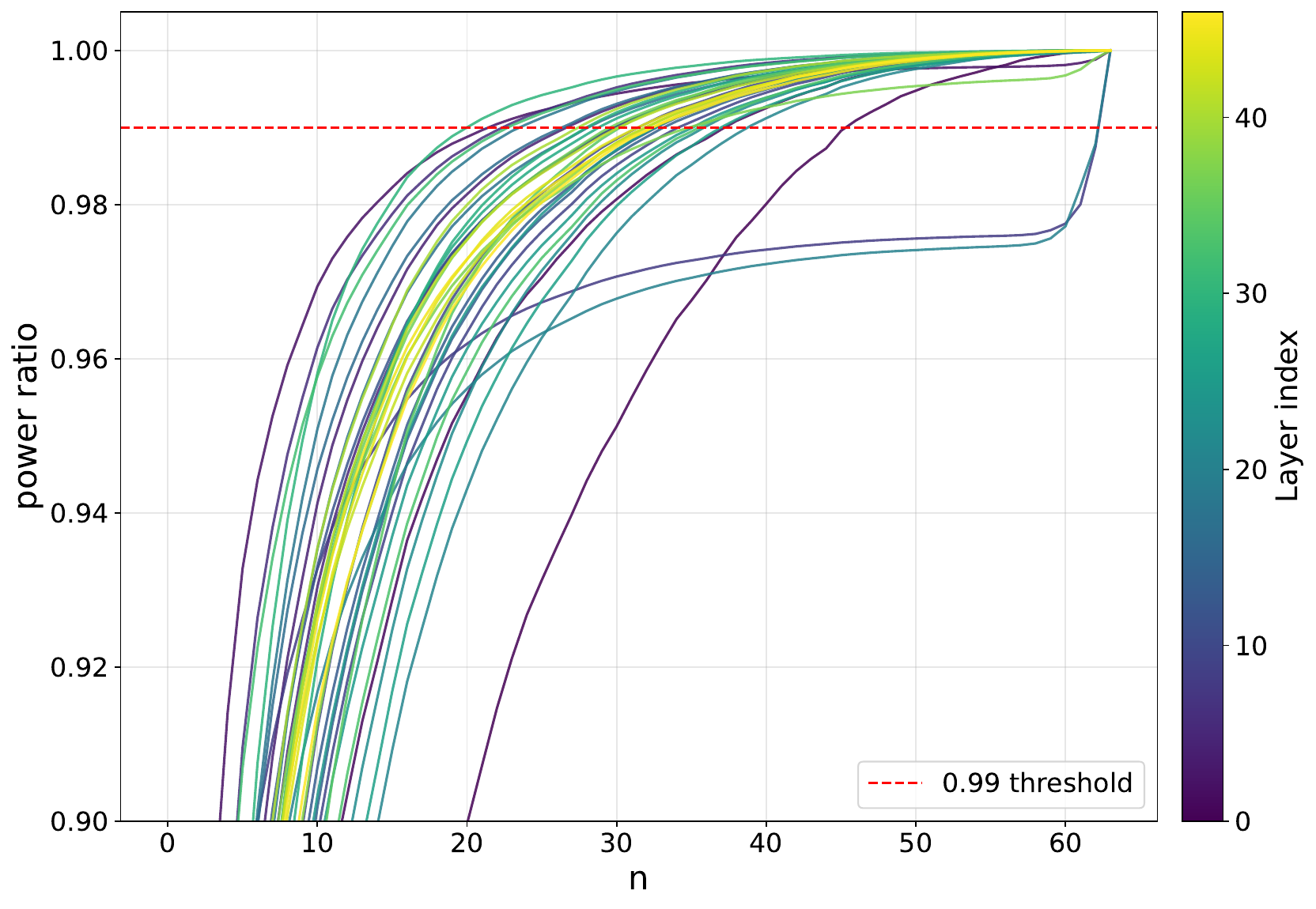}}
    \caption{
     Accumulated diagonal power ratio across layers. The ratio saturates quickly with increasing $n$; while >98\% is typically captured at small $n$, achieving $0.99$ requires substantially larger orders, revealing layer-wise variation and the cost–accuracy trade-off.
    }
    \label{fig:diagonal_power}
    \vskip -0.4 in
  \end{center}
\end{figure}

\paragraph{Real value experiments.} Based on Lemma~\ref{lemma:diagonal}, the truncation error can be efficiently estimated using the accumulated power ratio along the diagonal dimension. We empirically measure this accumulated power ratio on the Qwen3-NextA3B-80B~\cite{qwen3technicalreport} model across different linear attention layers.
As shown in Fig.~\ref{fig:quant_range}, the accumulated power ratio exhibits a consistent trend across layers: it increases rapidly with the order nnn and approaches saturation. In most layers, more than $98\%$ of the power is captured within a relatively small number of steps $(n \approx 10\text{--}20)$, while achieving a higher threshold (e.g., 0.99) typically requires significantly larger orders ($n \approx 30\text{--}50)$. Notably, there exists substantial layer-wise variation, with some layers converging more slowly and requiring much larger nnn to reach the same threshold.
These results indicate that, although the Neumann series converges in principle, achieving high-accuracy approximation uniformly across layers demands a large truncation order. This further reinforces the practical limitation of using high-order truncation, as it not only increases computational cost but also exacerbates numerical instability and quantization issues discussed earlier.

Although $(I-A)^{-1}$ admits an exact finite expansion for strictly lower-triangular $A$,
both our analysis (Lemma~\ref{lem:truncation_error}) and empirical validation (Figure~\ref{fig:truncation_k64})
show that a much lower truncation order already captures most of the inverse structure.
We therefore use a low-order truncated expansion as an initializer and recover the remaining tail error via residual correction.

\section{Proof of Lemma \ref{lemma:upper_bound}}
\begin{proof}
By expanding the matrix product,
\[
(A^k)_{ij}
=
\sum_{j < t_1 < \cdots < t_{k-1} < i}
A_{i,t_{k-1}} A_{t_{k-1},t_{k-2}} \cdots A_{t_1,j},
\]
where strict lower-triangularity enforces strictly increasing intermediate indices.
Since each factor satisfies $0 \le A_{pq} \le 1$, every product term is bounded above by $1$.
The number of admissible index sequences equals the number of ways to choose
$k-1$ indices from the $d-1 = i-j-1$ integers between $j$ and $i$, yielding
$\binom{d-1}{k-1}$.

Tightness follows by choosing $A_{pq} = 1$ for all $p > q$, in which case every product
term equals $1$ and the bound is attained with equality.
\end{proof}

\begin{corollary}[Worst-Case Characterization]
Under the assumptions of Lemma \ref{lemma:upper_bound},
\[
\max_{\,0 \le A_{ij} \le 1} (A^k)_{ij}
=
\binom{d-1}{k-1}.
\]
Thus, the binomial coefficient gives the exact worst-case growth of individual entries
of $A^k$, depending only on the diagonal distance $d=i-j$ and the power $k$.
\end{corollary}

\section{Proof of reformulation for Residual correction}
\begin{proof}
By definition, $(I-A)T^{(0)} = I - E$. Since $M$ is nonsingular and $\rho(E)<1$, the matrix $(I-E)$ is invertible and
\begin{equation}
T^{(0)} = (I-A)^{-1}(I-E)
\quad \Rightarrow \quad
(I-A)^{-1} = T^{(0)}(I-E)^{-1}.
\end{equation}
When $\rho(E)<1$, the Neumann series converges:
\begin{equation}
(I-E)^{-1} = \sum_{s=0}^{\infty} E^{s}.
\end{equation}
Thus $M^{-1} = T^{(0)}\sum_{s=0}^{\infty}E^s$. For truncation at $S$, the remainder is
\begin{equation}
M^{-1} - T^{(S)} = T^{(0)}\sum_{s=S+1}^{\infty} E^{s} = T^{(0)}E^{S+1}\sum_{t=0}^{\infty}E^{t}.
\end{equation}
Taking norms and using submultiplicativity gives
\begin{equation}
\|M^{-1}-T^{(S)}\|
\le
\|T^{(0)}\|\cdot \|E\|^{S+1}\sum_{t=0}^{\infty}\|E\|^{t}
=
\|T^{(0)}\|\cdot \frac{\|E\|^{S+1}}{1-\|E\|},
\end{equation}
which completes the proof.
\end{proof}

\section{Block-wise Matrix Inversion for Chunk-wise Processing}
\label{app:option2}

Let $T \in \mathbb{R}^{n \times n}$ be a strictly lower triangular matrix arising from chunk-wise Gated DeltaNet,
\begin{equation}
T = I - A, \qquad A_{ij} = 0 \;\; \forall\, i \le j .
\end{equation}

We partition $T$ into $K$ blocks of size $b \times b$:
\begin{equation}
T =
\begin{bmatrix}
T_{11} & 0 & \cdots & 0 \\
T_{21} & T_{22} & \cdots & 0 \\
\vdots & \vdots & \ddots & \vdots \\
T_{K1} & T_{K2} & \cdots & T_{KK}
\end{bmatrix}.
\end{equation}

The inverse preserves the same block lower-triangular structure,
\begin{equation}
T^{-1} =
\begin{bmatrix}
S_{11} & 0 & \cdots & 0 \\
S_{21} & S_{22} & \cdots & 0 \\
\vdots & \vdots & \ddots & \vdots \\
S_{K1} & S_{K2} & \cdots & S_{KK}
\end{bmatrix}.
\end{equation}

Each diagonal block is inverted independently:
\begin{equation}
S_{ii} = T_{ii}^{-1}, \qquad i = 1,\dots,K .
\end{equation}

For off-diagonal blocks $(i>j)$, block-wise forward substitution satisfies
\begin{equation}
\sum_{k=j}^{i} T_{ik} S_{kj} = 0 ,
\end{equation}
which yields the recursion
\begin{equation}
S_{ij} = - T_{ii}^{-1} \sum_{k=j}^{i-1} T_{ik} S_{kj},
\qquad i>j .
\end{equation}

In particular, the first off-diagonal block is given by
\begin{equation}
S_{i,i-1} = - T_{ii}^{-1} T_{i,i-1} T_{i-1,i-1}^{-1}.
\end{equation}

In chunk-wise Gated DeltaNet, each block corresponds to one chunk, and the output is computed as
\begin{equation}
Y = T^{-1} X ,
\end{equation}
using block-wise inversion (Eq.~4) and recursive accumulation (Eq.~6).

\section{Why Algebraic Factorization Does Not Help on NPUs.}
While factorizing the Neumann series reduces algebraic depth, it fundamentally alters the execution graph. The standard Neumann iteration computes powers of $A$ via the recurrence
\begin{equation}
X_{k+1} = X_k A,
\end{equation}
which forms a single linear dependency chain. Each iteration requires one matrix multiplication, maintains a single live accumulator, and admits kernel fusion across steps. This execution pattern is well aligned with NPU architectures, which are optimized for narrow, linear dependency graphs.
In contrast, factorized forms compute partial sums of the form
\begin{equation}
G_n = \sum_{i=0}^{2^n-1}X_{i}, G_n = G_{n-1}(I+A^{2^{(n-1)}}),
\end{equation}
where the term $I + A^{2^{(n-1)}}$ must be materialized prior to its multiplication with $G_{n-1}$. Consequently, each update introduces a two-stage dependency: first for power accumulation (to obtain $A^{2^k}$), and second for the subsequent multiplication. This effectively doubles the dependency width, increasing the number of live intermediate tensors and hindering fusion into a single accumulator chain.

On NPUs, such widened dependency graphs lead to higher scheduling overhead, increased memory traffic, and reduced pipeline utilization. Consequently, despite reduced algebraic depth, the factorized Neumann form incurs strictly higher execution cost than the unfactorized series.

\section{Quantization Experiment}
\label{app:quant_setting}
To further improve latency and power efficiency, we apply integer-only (INT) activation quantization to the chunk-wise Gated DeltaNet. We adopt a standard uniform quantization operator
\begin{equation}
    q(x)=clip(round(\frac{x}{\Delta}), qmin, qmax)
\end{equation}
where $\Delta$ denotes the quantization scale.
Depending on the hardware implementation, asymmetric quantization maps activations to the range $[0, 2^b - 1]$, while symmetric quantization uses a signed range of $[-2^{b-1}, 2^{b-1}-1]$, with $b$ indicating the bit width (e.g., $b=8$ or $16$).

\subsection{Quantization Settings}
We summarize the quantization setting in the following table.
For visual model, we do plain calibration with 25 image samples from LLaVA-COCO dataset~\cite{llava}. For text model, we use AdaScale from AIMET~\cite{aimet} with 1000 samples from C4~\cite{c4}.

\begin{table}[h]
\centering
\caption{
W4A16 quantization configuration for Qwen3.5-4B. Decoder weights are quantized to INT4 with per-channel scaling, while all activations are retained in INT16.
}
\begin{tabular}{lccc}
\toprule
\textbf{Module} & \textbf{Weight} & \textbf{Activation} & \textbf{Quantization Scheme} \\
\midrule
Decoder (LM) & INT4 & INT16 & Per-channel, symmetric  \\
Matrix Inversion & -- & INT16/INT8 & Per-tensor \\
Embedding & INT16 & INT16 & Per-tensor, asymmetric  \\
LM Head & INT8 & INT16 & Per-channel, symmetric  \\
Normalization & INT16 & INT16 & Per-tensor, asymmetric  \\
Vision Encoder (VLM) & INT8 & INT16 & Per-channel, symmetric  \\
Other Ops (Mul, Add) & -- & INT16 & Per-tensor, asymmetric/symmetrics  \\
\bottomrule
\end{tabular}
\label{tab:w4a16_qwen}
\end{table}

\subsection{Matrix inversion: INT16 vs. INT8}
\label{sec:w4a16_qwen}
The proposed method effectively addresses the numerical stability issues in matrix inversion, to maintain a compact dynamic range, making it well for INT quantization. In the experiment, we further evaluate INT8 setting for matrix inversion, where keeping all others consistent with the W4A16 setting. Table~\ref{tab:qwen35_int8_int16} compares INT16 and INT8 quantization for matrix inversion.
The results show that INT8 achieves comparable perplexity and downstream accuracy to INT16, demonstrating that our method remains robust even under lower-bit integer quantization.

\begin{table}[!thb]
    \setlength{\tabcolsep}{4pt}
    \renewcommand{\arraystretch}{0.95}
  \caption{INT8 VS INT16 comparison in matrix inversion, CS=128}
  \label{tab:qwen35_int8_int16}
  \begin{center}
    \begin{small}
      \begin{sc}
        \begin{tabular}{llcccc}
          \toprule
          Model & Method & Matrix inverse & PPL & IF-Eval\\
          \midrule
          \multirow{3}{*}{Qwen3.5 4B}   & FLA-FP   & FP16 &  8.89 & 0.83 \\
                                        & Ours-W4A16 & INT16  &  9.71 & 0.78 \\
                                        & Ours-W4A16 & INT8   &  9.71 & 0.77 \\
          \bottomrule
        \end{tabular}
      \end{sc}
    \end{small}
  \end{center}
  \vskip -0.1in
\end{table}

\section{Matrix inversion for 32$\times$32 matrix and 128 $\times$ 128 matrix}
For the $32 \times 32$ matrix, we further evaluate the impact of Neumann series order and iteration steps. As shown in Table~\ref{tab:order_step_ablation_cs_32}, the $32 \times 32$ matrix exhibits a similar trade-off to the $64 \times 64$ case. Due to its smaller spectral magnitude, the Neumann series for the $32 \times 32$ matrix converges more rapidly. Consequently, the optimal configuration of $S=3$ steps and $N=4$ order in FP16 also generalizes well to the INT16 setting.

\begin{table}[th]
  \caption{Effect of Neumann series order \(n\) and residual steps \(s\) on numerical stability. Baseline FP16 is 8.89.}
  \label{tab:order_step_ablation_cs_32}
  \begin{center}
      \begin{sc}
        \begin{tabular}{lcccc}
          \toprule
          & $N=1$ & $N=2$ & $N=3$ & $N=4$ \\
          \midrule
          $S=1$ & NaN   & NaN   & NaN      & 9.902 \\
          $S=2$ & NaN   & NaN   & 9.22 & 8.888   \\
          $S=3$ & NaN   & NaN   & 8.895  & \textbf{8.890}    \\
          $S=4$ & NaN   & 8.900 & \textbf{8.890}   & \textbf{8.890}  \\
          $S=5$ & NaN   & 8.891   & \textbf{8.890}   & \textbf{8.890}  \\
          $S=6$ & NaN & \textbf{8.890}   & \textbf{8.890}   & \textbf{8.890}  \\
          $S=7$ & NaN  & \textbf{8.890}   & \textbf{8.890}   & \textbf{8.890}  \\
          $S=8$ & 8.891 & \textbf{8.890} & \textbf{8.890} & \textbf{8.890} \\
          \bottomrule
        \end{tabular}
      \end{sc}
  \end{center}
  \vskip -0.1in
\end{table}

For the $128 \times 128$ matrix, according to Lemma~\ref{lemma:upper_bound}, even with $N=3$, the worst-case value can grow to $341{,}376$, which exceeds the dynamic range of both FP16 and INT16, leading to severe numerical instability. To mitigate this issue, we adopt a block-wise matrix inversion strategy, decomposing the problem into smaller sub-matrices with a minimum size of $64 \times 64$.

\section{Ablation on Different context length}
\label{app:cl_ablation}
We evaluate WikiText-v2 on Qwen3.5-4B across context lengths from 4k to 32k to assess the effectiveness of our method. In FP32, our approach exactly matches the FLA baseline across all context lengths, demonstrating numerical equivalence without any degradation. Under INT16 quantization, it achieves comparable performance at short contexts and consistent improvements at longer contexts.
At 32k, increasing the chunk size to CS=128 reduces perplexity from 9.201 to 8.998 ($-$0.203), significantly narrowing the FP32–INT16 gap. Notably, the impact of chunk size is context-dependent: while CS=128 slightly degrades performance at $\le$8k, it becomes beneficial beyond 16k, suggesting improved numerical stability in the long-context regime.

\begin{table}[thb]
\centering
\caption{Perplexity comparison across context lengths on WikiText-v2.
FP32 results show identical performance between na\"ive and our method.
Quantized results (INT16) demonstrate that our method consistently improves perplexity,
with larger chunk sizes (CS=128) yielding further gains at long context (e.g., 32k).}
\begin{tabular}{c|c|cc|cc}
\toprule
\multirow{2}{*}{Context} 
& \multirow{2}{*}{FP32} 
& \multicolumn{2}{c|}{FP32 (Ours)} 
& \multicolumn{2}{c}{Quantized (INT16)} \\
\cline{3-6}
& 
& CS=64 & CS=128 
& Na\"ive (CS=64) 
& Ours (CS=64 / 128) \\
\midrule
4k  & 8.89 & 8.89 & 8.89 & 9.658 & 9.664 / 9.717 \\
8k  & 8.66 & 8.66 & 8.66 & 9.373 & 9.366 / 9.405 \\
16k & 8.46 & 8.46 & 8.46 & 9.216 & 9.197 / 9.179 \\
32k & 8.41 & 8.41 & 8.41 & 9.201 & 9.188 / 8.998 \\
\bottomrule
\end{tabular}

\label{tab:ppl_context}
\end{table}

\end{document}